\documentclass{article}

\usepackage{arxiv}

\usepackage[utf8]{inputenc} 
\usepackage[T1]{fontenc}    
\usepackage{hyperref}       
\usepackage{url}            
\usepackage{booktabs}       
\usepackage{amsfonts}       
\usepackage{nicefrac}       
\usepackage{microtype}      
\usepackage{lipsum}
\usepackage{graphicx}
\graphicspath{ {./images/} }
\usepackage{graphicx,float}
\usepackage{amsmath,amssymb}
\usepackage{amsthm}
\theoremstyle{plain}
\newtheorem{theorem}{Theorem}[section]
\newtheorem{lemma}[theorem]{Lemma}

\theoremstyle{definition}
\newtheorem{definition}[theorem]{Definition}

\usepackage{tikz}
\usetikzlibrary{positioning,arrows.meta,patterns}
\definecolor{connectColor}{RGB}{86,155,189}
\definecolor{blockColor}{RGB}{235,235,235}
\definecolor{orangeColor}{RGB}{255,165,0}

\newcommand{\mask}[2]{%
  \pgfmathsetlengthmacro{\matrixWidth}{5cm}
  \pgfmathsetlengthmacro{\cellSize}{\matrixWidth / #2}
  \begin{tikzpicture}[x=\cellSize, y=\cellSize]
    \def\n{#2}
    \foreach [count=\idx] \val in {#1} {
      \pgfmathtruncatemacro{\i}{int(ceil(\idx/\n))}
      \pgfmathtruncatemacro{\j}{int(\idx - (\i-1)*\n)}
      \ifnum \val=1
        \fill[connectColor] (\j-1, \n-\i) rectangle ++(1,1);
      \else
        \ifnum \val=2
          \fill[orangeColor] (\j-1, \n-\i) rectangle ++(1,1);
        \else
          \ifnum \val=3
            \fill[blockColor] (\j-1, \n-\i) rectangle ++(1,1);
            \fill[pattern=north west lines, pattern color=connectColor]
                 (\j-1, \n-\i) rectangle ++(1,1);
          \else
            \fill[blockColor] (\j-1, \n-\i) rectangle ++(1,1);
          \fi
        \fi
      \fi
      \draw[gray!50, thin] (\j-1, \n-\i) rectangle ++(1,1);
    }
    \ifnum \n<9
      \foreach \i in {1,...,\n}{
        \node[left=0.3cm] at (0, \n-\i+0.5) {$\i$};
      }
      \foreach \j in {1,...,\n}{
        \node[above=0.15cm] at (\j-0.5, \n) {$\j$};
      }
    \fi
  \end{tikzpicture}%
}

\title{Hasse Diagrams for Attention: A Partial Order Framework for Designing Transformer Masks}

\author{
 Chentao Li \\
  School of Mathematics and Statistics\\
  Northeast Petroleum University\\
  Daqing, Heilongjiang 163318, China \\
  \texttt{li\_chentao \@foxmail.com} \\
   \And
  Han Guo \\
  School of Mathematics and Statistics\\
  Northeast Petroleum University\\
  Daqing, Heilongjiang 163318, China \\
  \texttt{zhihui75888  \@163.com} \\
  }

\begin{document}
\maketitle
\begin{abstract}
During the training of large Transformer models, attention masks regulate the scope and direction of information flow across a sequence. Numerous mask variants exist, and operators such as FlexAttention already support arbitrary attention masks. Nevertheless, a systematic formal analysis of the information-flow structure induced by arbitrary masks has been missing. This paper develops a complete theoretical framework. We prove that, with sufficient depth, the information flow of a multi-layer Transformer converges to a Hasse diagram—a directed acyclic graph representing a partial order. Building on this, we recast the design of parallel training tasks as the problem of finding a minimal common supergraph of Hasse diagrams, and we establish a criterion for the minimal common supergraph. This yields a constructive method to derive attention masks directly from a family of tasks. Applying the framework, we design two novel masks: a block-generation attention mask that ensures training–inference consistency (Block Two-Stream Attention), and a fully supervised bidirectional attention mask (Butterfly Attention). These results demonstrate the framework's capacity to discover new structures.
\end{abstract}


\section{Introduction}
The attention mechanism lies at the core of the Transformer architecture, endowing each output position with the ability to aggregate global information, while attention masks control the direction of information flow.

Within the literature on attention masks, one line of work focuses on sparsifying causal masks. Iz Beltagy et al.\ \cite{DBLP:journals/corr/abs-2004-05150} proposed sliding-window attention, and Heng\ \cite{zhang2025dam} introduced a dynamic sparse attention mechanism. To avoid hand-crafted acceleration kernels for every sparse pattern, Dong et al.\ \cite{FlexAttention} realized a FlexAttention operator that supports arbitrary attention masks. 

At the same time, the design of an attention mask not only governs information flow but also requires the input sequence and the supervision target to be adapted accordingly. For instance, autoregressive language modeling\ \cite{Pre-Training} employs a causal mask, so that each position can attend only to itself and earlier tokens, predicting the next token step by step. Masked language modeling, represented by BERT\ \cite{BERT}, uses a fully connected mask together with mask tokens, enabling the model to predict masked positions based on bidirectional context. Yang et al.\ proposed two-stream attention (XLNet)\ \cite{XLNet}, which, for the first time, allowed the input sequence to be longer than the data sequence, showing that tokens can appear multiple times and thus offer a larger design space. Although the original implementation did not rely on a pure mask construction, it can be realized by concatenating the two streams into a long sequence and building a mask on that sequence; hence it also falls within the scope of attention mask design.

Despite the considerable results obtained by analyzing Transformers through graph theory—treating the attention mask as a reachability matrix being a standard approach—this perspective, when applied to arbitrary masks, yields only arbitrary directed graphs, which is too broad. Consequently, no prior work has successfully extracted mathematical structure from an arbitrary attention mask. Moreover, although two-stream attention broke conceptual limitations and inspired many subsequent designs, no systematic design methodology was proposed.

To address these issues, this paper extends the object of analysis from the attention mask to the Transformer block. By incorporating residual connections, we add self-loops to the reachability matrix, so that, after stacking multiple layers, the reachability matrix becomes not only transitive but also reflexive. This exactly satisfies the conditions of a preorder. From the preorder we naturally derive equivalence classes, and then induce a partial order on the quotient set; the Hasse diagram of this partial order is precisely the information-flow graph of the model.

On this foundation, we formally reduce the problem of designing training tasks for attention mechanisms to the minimal-common-supergraph problem. We provide a criterion for the minimal common supergraph, offering a general method for designing attention masks and training schemes directly from a family of tasks. We then design two entirely new attention mechanisms: Butterfly Attention and Block Two-Stream Attention. Butterfly Attention achieves bidirectional attention and 100\% supervision density without granting any position access to its own token. Block Two-Stream Attention realizes, for the first time, block-generation training without approximation. These constructions validate the design capability of our framework.

\section{Characterization}
\label{sec:graph_and_reachability}

This section analyzes a multi-layer Transformer with an arbitrary mask and formalizes its information-flow graph as a Hasse diagram.

A standard approach treats an attention mask as the adjacency matrix of a directed graph. An arbitrary mask, however, merely describes an arbitrary directed graph, from which no useful mathematical structure can be extracted. This section shifts the object of analysis from the attention mechanism to the entire Transformer block. This extension brings residual connections into the picture, endowing each node of the directed graph with a self-loop and thereby supplying reflexivity. At the same time, stacking multiple Transformer blocks induces transitivity.

With reflexivity and transitivity in place, one can apply the standard mathematical machinery that leads from preorders to partial orders and finally to Hasse diagrams, yielding the information-flow graph of a Transformer under an arbitrary mask.

\subsection{Single-layer and Multi-layer Transformer Information Flow Graph Modeling}

In a Transformer, the attention mechanism together with residual connections determines the paths along which information flows among different positions in a sequence. To analyze this flow rigorously, we model the computation of each layer as a bipartite directed graph, and from this we construct the graph model of the entire multi-layer network.

\begin{definition}[Single-Layer Information Flow Graph]
\label{def:single_layer_graph}
Let the input to the $l$-th Transformer block be the sequence of representations $X = [x_1; \dots ; x_n]\in\mathbb{R}^{n\times d}$, and let the output be $Y = [y_1; \dots ; y_n]\in\mathbb{R}^{n\times d}$. Denote the set of input positions by $U = \{u_1,\dots,u_n\}$ and the set of output positions by $V = \{v_1,\dots,v_n\}$. The information flow of this layer is characterized by a bipartite directed graph $G = (U,V,E)$, where the edge set $E\subseteq U\times V$ satisfies:
\begin{equation}
    (u_i, v_j) \in E \;\Longleftrightarrow\; \text{information at position $i$ can flow to the output at position $j$}.
\end{equation}

The structure of $G$ is equivalent to an adjacency matrix $A\in\{0,1\}^{n\times n}$ defined by:
\begin{equation}\label{eq:adj_def}
    (u_i, v_j) \in E \Longleftrightarrow A_{i,j} = 1.
\end{equation}

In standard Transformer implementations, the edge set $E$ is jointly determined by the attention mask $M\in\{0,-\infty\}^{n\times n}$ and the residual connections. The attention computation is:
\begin{equation}
    \text{Attention}(X) = \text{softmax}\!\left(\frac{QK^\top}{\sqrt{d_k}} + M\right)V,
\end{equation}
where $Q = XW_Q$, $K = XW_K$, $V = XW_V$. A mask value $M_{ij}=0$ allows position $j$ to attend to position $i$, whereas $M_{ij}=-\infty$ forbids it. In addition, the residual connection provides each position with a direct edge from input to output. Combining these factors, the adjacency matrix $A$ can be expressed explicitly as:
\begin{equation}\label{eq:A_from_M}
    A_{i,j} = 
    \begin{cases}
        1, & \text{if } M_{j,i} = 0 \text{ or } i = j,\\
        0, & \text{if } M_{j,i} = -\infty.
    \end{cases}
\end{equation}
Here $M$ takes only values $0$ or $-\infty$, and the residual connection guarantees $A_{i,i}=1$ for all $i$. Because the feed-forward layers and normalization layers are position-wise operations, they do not affect the information-flow graph and are omitted in this modeling.
\end{definition}
In the subsequent Boolean matrix operations, we adopt the rules of Boolean algebra: logical addition $\lor$ ($1\lor1=1$, $1\lor0=1$, $0\lor0=0$) and logical multiplication $\land$ ($1\land1=1$, $1\land0=0$, $0\land0=0$). For brevity, Boolean matrix multiplication is denoted by ``$\times$'', i.e.\ $(\mathbf{B}\times\mathbf{C})_{i,j} = \bigvee_k \mathbf{B}_{i,k}\land\mathbf{C}_{k,j}$.

\begin{definition}[Multi-Layer Information Flow Graph]
\label{def:multi_layer_graph}
Stack $L$ structurally identical single-layer blocks. Let the bipartite directed graph of layer $l$ be $G^{(l)} = (U^{(l)}, V^{(l)}, E^{(l)})$, where adjacent layers satisfy $V^{(l-1)} = U^{(l)}$ ($l=1,\dots,L-1$) and all layers share the same adjacency matrix $A$ (derived from the same mask $M$ according to Definition~\ref{def:single_layer_graph}). The embedding layer and the output head are also position-wise and do not change the information-flow graph.

The information flow of the whole $L$-layer Transformer forms a multi-layer bipartite directed graph $\mathcal{G} = (\mathcal{V}, \mathcal{E})$ with vertex and edge sets:
\begin{equation}
    \mathcal{V} = V^{(L)} \cup \bigcup_{l=0}^{L} U^{(l)},\quad
    \mathcal{E} = \bigcup_{l=1}^{L} E^{(l)}.
\end{equation}
Since all edges point from lower layers to higher layers, $\mathcal{G}$ is a strictly layered acyclic graph.
\end{definition}

To further characterize the reachability from initial inputs to final outputs, we introduce the notion of a reachability matrix.

\begin{definition}[Reachability Matrix]
\label{def:reachability}
In $\mathcal{G}$, collect the reachability of all input–output pairs into a reachability matrix $R^{(L)}\in\{0,1\}^{n\times n}$ with elements:
\begin{equation}
    R^{(L)}_{i,j} = 1 \;\Longleftrightarrow\; u_i^0 \text{ can reach } v_j^L.
\end{equation}
Since every layer has the same adjacency matrix and information flows only forward through the layers, Boolean matrix multiplication directly yields:
\begin{equation}\label{eq:R_power}
    R^{(L)} = A^L = \underbrace{A\times A\times\cdots\times A}_{L\text{ times}}.
\end{equation}
\end{definition}

\subsection{Monotonicity and Convergence of the Reachability Matrix}

In a multi-layer Transformer, as the number of layers $L$ increases, the reachability matrix acquires more connections but eventually converges to a stable structure. We now prove this fact rigorously.

\begin{lemma}[Monotonicity]
\label{lem:monotonicity}
For any $L\ge1$ and all $i,j\in\{1,\dots,n\}$,
\begin{equation}
   (R^{(L)})_{i,j} \le (R^{(L+1)})_{i,j},
\end{equation}
where $\le$ is understood in the Boolean sense: if $(R^{(L)})_{i,j}=1$, then necessarily $(R^{(L+1)})_{i,j}=1$.
\end{lemma}

\begin{proof}
From $R^{(L)} = A^L$,
\[
    (A^{L+1})_{i,j} = \bigvee_{k=1}^{n} (A^L)_{i,k} \land A_{k,j}.
\]
Taking the term with $k=j$ gives
\[
    (A^{L+1})_{i,j} \ge (A^L)_{i,j} \land A_{j,j}.
\]
Residual connections ensure $A_{j,j}=1$, hence
\[
    (A^{L+1})_{i,j} \ge (A^L)_{i,j} \land 1 = (A^L)_{i,j},
\]
i.e., $(R^{(L+1)})_{i,j} \ge (R^{(L)})_{i,j}$.
\end{proof}

Let $N(L)$ be the number of entries in $R^{(L)}$ equal to $1$. Clearly $0\le N(L)\le n^2$. As $L$ increases, $N(L)$ is nondecreasing and bounded above, so it must eventually stop changing. The following lemma shows that when this happens, the matrix has converged.

\begin{lemma}[Sufficient Condition for Convergence]
\label{lem:convergence_condition}
If there exists $k\ge1$ such that $N(k+1)=N(k)$, then $R^{(k+1)} = R^{(k)}$, and for all $m\ge k$, $R^{(m)} = R^{(k)}$.
\end{lemma}

\begin{proof}
Assume $N(k+1)=N(k)$ but $R^{(k+1)}\neq R^{(k)}$. By Lemma~\ref{lem:monotonicity}, no entry can decrease, so there must exist a pair $(i,j)$ with $(R^{(k+1)})_{i,j} > (R^{(k)})_{i,j}$, i.e., a change from $0$ to $1$. This would imply $N(k+1)\ge N(k)+1$, contradicting $N(k+1)=N(k)$. Hence $R^{(k+1)} = R^{(k)}$. For any $t\ge1$, $R^{(k+t)} = A^{k+t} = A^{k+t-1}\times A$, and using $A^{k}=R^{(k)}$ together with $A^{k+1}=A^{k}$ one proves by induction that $R^{(k+t)} = R^{(k)}$.
\end{proof}

\begin{theorem}[Convergence of the Reachability Matrix]
\label{thm:convergence}
The limit $R = \lim_{L\to\infty} R^{(L)}$ exists and is unique, and it satisfies:
\begin{equation}\label{eq:limit_eq}
    R = R \times A = R \times R.
\end{equation}
\end{theorem}

\begin{proof}
The sequence $\{N(L)\}_{L=1}^{\infty}$ takes values in the finite set $\{0,1,\dots,n^2\}$. By the pigeonhole principle, there exists $L_0$ such that $N(L_0+1)=N(L_0)$. By Lemma~\ref{lem:convergence_condition}, for all $L\ge L_0$, $R^{(L)} = R^{(L_0)}$; thus the limit exists and equals this constant matrix $R$. Equation~\eqref{eq:limit_eq} follows directly from Definition~\ref{def:reachability}.
\end{proof}

Theorem~\ref{thm:convergence} shows that, as long as the Transformer has enough layers, the information-flow structure necessarily converges to a limiting reachability matrix $R$. If $R = R^{(1)} = A$, i.e., convergence occurs within a single layer, we call the mask a dense attention mask; the sparser the mask, the more layers are needed for convergence. In the subsequent analysis we assume the network depth is sufficient for the reachability matrix to have reached its limit $R$, and we refer to $R$ directly as the reachability matrix of the Transformer. Next we introduce a partial order on position indices based on $R$, and visualize the entire information-flow structure using a Hasse diagram.

\subsection{Partial-Order Structure of Information Flow and the Hasse Diagram}
\label{sec:partial_order}

In the previous section we proved that, with increasing depth, the reachability matrix of a Transformer converges to a limit $R$. We now translate the reachability relation encoded by $R$ into a partial-order structure on the position indices and visualize it with a Hasse diagram.

\begin{definition}[Query Relation]
\label{def:query_relation}
On the set of position indices $S = \{1,2,\dots,n\}$, define a binary relation $\preceq$ from the limiting reachability matrix $R$:
\begin{equation}
    j \preceq i \quad\Longleftrightarrow\quad R_{i,j} = 1.
\end{equation}
Intuitively, when the network depth is sufficiently large, the information of input position $j$ can reach the output position $i$ through the intermediate layers. We say that $j$ is \emph{exposed} to $i$, or $i$ \emph{can query} $j$. If $A_{i,j}=1$ (i.e., a single-layer connection exists), we say $i$ \emph{directly queries} $j$; otherwise $i$ \emph{indirectly queries} $j$.
\end{definition}

\begin{lemma}[Preorder Property]
\label{lem:preorder}
The relation ``$\preceq$'' is a \emph{preorder}, i.e., it satisfies:
\begin{enumerate}
    \item \emph{Reflexivity}: for every $i\in S$, $i \preceq i$;
    \item \emph{Transitivity}: if $j \preceq k$ and $k \preceq i$, then $j \preceq i$.
\end{enumerate}
\end{lemma}

\begin{proof}
Reflexivity follows from residual connections: by Definition~\ref{def:single_layer_graph}, residual connections guarantee $A_{i,i}=1$, and by monotonicity (Lemma~\ref{lem:monotonicity}), $(R^{(L)})_{i,i}=1$ for all $L\ge1$. Taking the limit yields $R_{i,i}=1$, hence $i \preceq i$.

Transitivity follows from Theorem~\ref{thm:convergence}. By Boolean matrix multiplication,
\[
    R_{i,j} = \bigvee_{k=1}^{n} R_{i,k} \land R_{k,j}.
\]
Thus if $j\preceq k$ ($R_{k,j}=1$) and $k\preceq i$ ($R_{i,k}=1$), the term with $k$ gives $R_{i,j} \ge R_{i,k} \land R_{k,j} = 1$, so $R_{i,j}=1$, i.e., $j\preceq i$.
\end{proof}

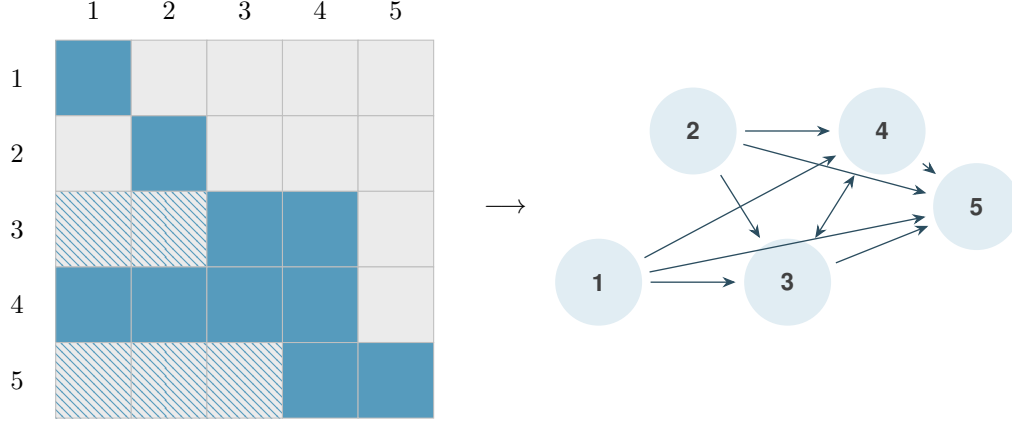
\begin{figure}[H]
    \centering
    $\vcenter{\hbox{%
        \begin{tikzpicture}
            \mask{
            1,0,0,0,0,
            0,1,0,0,0,
            3,3,1,1,0,
            1,1,1,1,0,
            3,3,3,1,1
            }{5}
        \end{tikzpicture}%
    }}$
    \hspace{1.2cm}
    $\longrightarrow$
    \hspace{0.2cm}
    $\vcenter{\hbox{%
        \begin{tikzpicture}[
        >={Stealth[length=5pt, width=4pt]},
        vertex/.style={
            circle,
            fill=connectColor!18,
            draw=none,
            minimum size=1.15cm,
            inner sep=0pt,
            font=\sffamily\bfseries\small,
            text=black!75
        },
        edge/.style={
            ->,
            line width=0.5pt,
            color=connectColor!50!black,       
            shorten <=3pt,
            shorten >=3pt
        },
        biedge/.style={
            <->,
            line width=0.5pt,
            color=connectColor!50!black,
            shorten <=3pt,
            shorten >=3pt
        }
        ]
        \node[vertex] (1) at (0,0)    {1};
        \node[vertex] (2) at (1.25,2) {2};
        \node[vertex] (3) at (2.5,0)  {3};
        \node[vertex] (4) at (3.75,2) {4};
        \node[vertex] (5) at (5,1) {5};

        \draw[edge] (1) -- (3);
        \draw[edge] (2) -- (3);
        \draw[edge] (2) -- (4);
        \draw[biedge] (3) -- (4);  
        \draw[edge] (1) -- (4);   
        \draw[edge] (1) -- (5);
        \draw[edge] (2) -- (5);
        \draw[edge] (3) -- (5);
        \draw[edge] (4) -- (5);
        \end{tikzpicture}%
    }}$
    \caption{Preorder relation defined by $R$ on the indices. In the mask matrix, solid blue blocks indicate the original matrix $A$, and hatched blocks indicate connections added by $R$ beyond $A$. An arrow $i\rightarrow j$ means information flows from $i$ to $j$.}
    \label{fig:mask-to-graph}
\end{figure}

Figure~\ref{fig:mask-to-graph} shows a reachability matrix $R$ and its corresponding preorder relation $\preceq$.

We can now use the preorder ``$\preceq$'' to define a partial order on the positions.

\begin{definition}[Equivalence Relation and Quotient Set]
\label{def:equivalence}
Define a relation $\sim$ on $S$ by:
\begin{equation}
    i \sim j \quad\Longleftrightarrow\quad i \preceq j \text{ and } j \preceq i.
\end{equation}
Clearly $\sim$ is an equivalence relation. Denote the equivalence class of $i$ by $c_{[i]}$:
\[
    c_{[i]} = \{k \in S \mid k \sim i\}.
\]
The set of all equivalence classes is the quotient set $\mathcal{S} = S /{\sim} = \{c_{[i]} \mid i \in S\}$.
\end{definition}

Within an equivalence class $c_{[i]}$, the input of any position can influence the output of every position in the class, and this relation is symmetric. Hence, for computation, these positions must be treated as a whole.

We now induce a partial order on the quotient set $\mathcal{S}$.

\begin{theorem}[Induced Partial Order]
\label{thm:induced_partial_order}
On $\mathcal{S}$ define a binary relation $\le$ by:
\begin{equation}
    c_{[i]} \le c_{[j]} \quad\Longleftrightarrow\quad i \preceq j.
\end{equation}
Then:
\begin{enumerate}
    \item $\le$ is well defined (independent of the choice of representatives);
    \item $(\mathcal{S},\le)$ forms a partially ordered set (poset), satisfying reflexivity, antisymmetry, and transitivity.
\end{enumerate}
We call $\le$ the \emph{query order}. Intuitively, $c_{[i]} \le c_{[j]}$ means that the key-value information of equivalence class $c_{[i]}$ is used when computing $c_{[j]}$; $c_{[j]}$ can query $c_{[i]}$.
\end{theorem}

\begin{proof}
\textbf{(1) Well-definedness:} Let $i \sim i'$ and $j \sim j'$. We must show $i \preceq j \iff i' \preceq j'$. From $i \sim i'$ we have $i \preceq i'$ and $i' \preceq i$; similarly $j \sim j'$ gives $j \preceq j'$ and $j' \preceq j$. If $i \preceq j$, then $i' \preceq i \preceq j \preceq j'$, and by transitivity $i' \preceq j'$. The converse is analogous. Hence the definition is independent of representatives.

\textbf{(2) Partial order:}
\begin{itemize}
    \item \textbf{Reflexivity:} $i \preceq i$ gives $c_{[i]} \le c_{[i]}$.
    \item \textbf{Antisymmetry:} If $c_{[i]} \le c_{[j]}$ and $c_{[j]} \le c_{[i]}$, then $i \preceq j$ and $j \preceq i$, so $i \sim j$ and $c_{[i]} = c_{[j]}$.
    \item \textbf{Transitivity:} If $c_{[i]} \le c_{[j]}$ and $c_{[j]} \le c_{[k]}$, then $i \preceq j$ and $j \preceq k$; by transitivity of $\preceq$ (Lemma~\ref{lem:preorder}), $i \preceq k$, hence $c_{[i]} \le c_{[k]}$.
\end{itemize}
Thus $(\mathcal{S},\le)$ is a poset.
\end{proof}

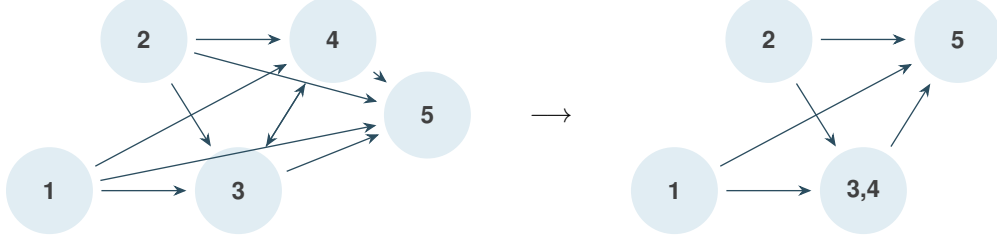
\begin{figure}[H]
    \centering
      $\vcenter{\hbox{%
        \begin{tikzpicture}[
        >={Stealth[length=5pt, width=4pt]},
        vertex/.style={
            circle,
            fill=connectColor!18,
            draw=none,
            minimum size=1.15cm,
            inner sep=0pt,
            font=\sffamily\bfseries\small,
            text=black!75
        },
        edge/.style={
            ->,
            line width=0.5pt,
            color=connectColor!50!black,       
            shorten <=3pt,
            shorten >=3pt
        },
        biedge/.style={
            <->,
            line width=0.5pt,
            color=connectColor!50!black,
            shorten <=3pt,
            shorten >=3pt
        }
        ]
        \node[vertex] (1) at (0,0)    {1};
        \node[vertex] (2) at (1.25,2) {2};
        \node[vertex] (3) at (2.5,0)  {3};
        \node[vertex] (4) at (3.75,2) {4};
        \node[vertex] (5) at (5,1) {5};

        \draw[edge] (1) -- (3);
        \draw[edge] (2) -- (3);
        \draw[edge] (2) -- (4);
        \draw[edge] (3) -- (4);   
        \draw[edge] (1) -- (4);   
        \draw[edge] (4) -- (3); 
        \draw[edge] (1) -- (5);
        \draw[edge] (2) -- (5);
        \draw[edge] (3) -- (5);
        \draw[edge] (4) -- (5);
        \end{tikzpicture}%
    }}$
    \hspace{0.6cm}
    $\longrightarrow$
    \hspace{0.6cm}
      $\vcenter{\hbox{%
        \begin{tikzpicture}[
        >={Stealth[length=5pt, width=4pt]},
        vertex/.style={
            circle,
            fill=connectColor!18,
            draw=none,
            minimum size=1.15cm,
            inner sep=0pt,
            font=\sffamily\bfseries\small,
            text=black!75
        },
        edge/.style={
            ->,
            line width=0.5pt,
            color=connectColor!50!black,       
            shorten <=3pt,
            shorten >=3pt
        },
        biedge/.style={
            <->,
            line width=0.5pt,
            color=connectColor!50!black,
            shorten <=3pt,
            shorten >=3pt
        }
        ]
        \node[vertex] (1) at (0,0)    {1};
        \node[vertex] (2) at (1.25,2) {2};
        \node[vertex] (3) at (2.5,0)  {3,4};
        \node[vertex] (4) at (3.75,2) {5};

        \draw[edge] (1) -- (3);
        \draw[edge] (2) -- (3);
        \draw[edge] (3) -- (4);
        \draw[edge] (1) -- (4);
        \draw[edge] (2) -- (4);
        \end{tikzpicture}%
    }}$
    \caption{Preorder relation and the induced partial order.}
    \label{fig:预序关系及偏序关系图}
\end{figure}

Figure~\ref{fig:预序关系及偏序关系图} shows a preorder and the corresponding partial order. However, when there are many nodes, such a partial-order diagram is still not intuitive enough; we therefore simplify it using a Hasse diagram.

\begin{definition}[Hasse Diagram]
\label{def:hasse}
The Hasse diagram of the poset $(\mathcal{S},\le)$ is a directed acyclic graph $\mathcal{H} = (\mathcal{S}, \mathcal{E}_{\mathcal{H}})$ whose vertex set is $\mathcal{S}$. The edge set $\mathcal{E}_{\mathcal{H}}$ is defined by the covering relation:
\begin{equation}
    (c_{[i]}, c_{[j]}) \in \mathcal{E}_{\mathcal{H}} \;\Longleftrightarrow\;
    c_{[i]} < c_{[j]} \text{ and there exists no } c_{[k]} \in \mathcal{S} \text{ such that } c_{[i]} < c_{[k]} < c_{[j]},
\end{equation}
where $c_{[i]} < c_{[j]}$ means $c_{[i]} \le c_{[j]}$ and $c_{[i]} \neq c_{[j]}$. An edge $(c_{[i]}, c_{[j]})$ indicates that $c_{[i]}$ directly influences $c_{[j]}$ and this dependency cannot be further decomposed.
\end{definition}

The Hasse diagram $\mathcal{H}$ captures the direction of information flow in the Transformer: each node represents an equivalence class of mutually reachable positions, and edges represent dependencies between classes. The computation of any node can access the key-value information of all its upstream nodes while remaining causally isolated from sibling and downstream nodes.

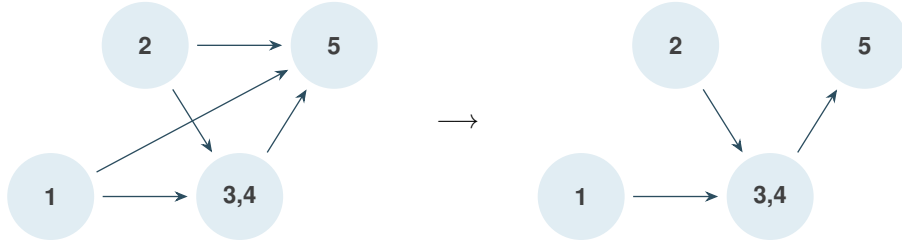
\begin{figure}[H]
    \centering
      $\vcenter{\hbox{%
        \begin{tikzpicture}[
        >={Stealth[length=5pt, width=4pt]},
        vertex/.style={
            circle,
            fill=connectColor!18,
            draw=none,
            minimum size=1.15cm,
            inner sep=0pt,
            font=\sffamily\bfseries\small,
            text=black!75
        },
        edge/.style={
            ->,
            line width=0.5pt,
            color=connectColor!50!black,       
            shorten <=3pt,
            shorten >=3pt
        },
        biedge/.style={
            <->,
            line width=0.5pt,
            color=connectColor!50!black,
            shorten <=3pt,
            shorten >=3pt
        }
        ]
        \node[vertex] (1) at (0,0)    {1};
        \node[vertex] (2) at (1.25,2) {2};
        \node[vertex] (3) at (2.5,0)  {3,4};
        \node[vertex] (4) at (3.75,2) {5};

        \draw[edge] (1) -- (3);
        \draw[edge] (2) -- (3);
        \draw[edge] (3) -- (4);
        \draw[edge] (1) -- (4);
        \draw[edge] (2) -- (4);
        \end{tikzpicture}%
    }}$
    \hspace{0.6cm}
    $\longrightarrow$
    \hspace{0.6cm}
      $\vcenter{\hbox{%
        \begin{tikzpicture}[
        >={Stealth[length=5pt, width=4pt]},
        vertex/.style={
            circle,
            fill=connectColor!18,
            draw=none,
            minimum size=1.15cm,
            inner sep=0pt,
            font=\sffamily\bfseries\small,
            text=black!75
        },
        edge/.style={
            ->,
            line width=0.5pt,
            color=connectColor!50!black,       
            shorten <=3pt,
            shorten >=3pt
        },
        biedge/.style={
            <->,
            line width=0.5pt,
            color=connectColor!50!black,
            shorten <=3pt,
            shorten >=3pt
        }
        ]
        \node[vertex] (1) at (0,0)    {1};
        \node[vertex] (2) at (1.25,2) {2};
        \node[vertex] (3) at (2.5,0)  {3,4};
        \node[vertex] (4) at (3.75,2) {5};

        \draw[edge] (1) -- (3);
        \draw[edge] (2) -- (3);
        \draw[edge] (3) -- (4);
        \end{tikzpicture}%
    }}$
    \caption{Partial order and the corresponding Hasse diagram.}
    \label{fig:偏序关系及对应的哈斯图}
\end{figure}

Figure~\ref{fig:偏序关系及对应的哈斯图} illustrates a Hasse diagram obtained by simplifying a partial order. Note that for later figures with many nodes we will not adhere to the standard Hasse diagram drawing convention. In the next section, we map concrete input data, output representations, and training labels onto the nodes of the Hasse diagram, thereby transforming the problem of designing parallel training tasks into the problem of merging Hasse diagrams.

\section{Construction}

Building on the Hasse diagram defined in Section 2, this section views a single training run as the merged execution of multiple subtasks. From this perspective we revisit the problem of designing training schemes for attention mechanisms and provide a systematic solution.

For simplicity, we begin with the case of dense attention masks. We first formalize the notions of a training task and a dense training task, then formulate a condition for merging nodes, and based on this condition define the minimal merged task. Training with this minimal merged task is equivalent to simultaneously training every task in the family. The sparse case is handled by subsequent sparsification of the minimal merged task.

Under this viewpoint, as long as an arbitrary training objective is given—or a task family is clearly specified—the minimal merged task of that family can be derived constructively. Consequently, when designing a new attention mechanism, one no longer needs to inspect, on a per-node basis, whether information leakage occurs or whether the supervision target is satisfied for the entire training task. This considerably simplifies the design of attention mechanism training procedures.

\subsection{Tasks and Their Hasse Diagram Representation}
\label{sec:task_and_hasse}

\begin{definition}[Training Task]
\label{def:training_task}
Let $\mathcal{D}$ be a dataset in which every element has been embedded together with positional encoding. For a given sample, let $C$ be the set of all tokens that appear, $|C|$ the total number of tokens, and $S_C = \{1,2,\dots,|C|\}$ its index set.

A training task is defined on a subset of $S_C$ and is a triple:
\begin{equation}
    \tau = (\mathcal{W}, \mathcal{V}, M),
\end{equation}
where:
\begin{itemize}
    \item $\mathcal{W}: S \to \mathcal{D}$ is the input sequence mapping, $S = \{1,2,\dots,k\}$ is the index set of $\tau$, $k \le |C|$. $\mathcal{W}(a)$ denotes the input token at position $a$ of $\tau$.
    \item $\mathcal{V}: S_{\mathcal{V}} \to \mathcal{D}$ is the label mapping, with domain $S_{\mathcal{V}} \subseteq S$. For $a \in S_{\mathcal{V}}$, $\mathcal{V}(a)$ gives the ground-truth label at position $a$.
    \item $M \in \{0,-\infty\}^{k \times k}$ is the attention mask; $M_{i,j}=0$ means position $j$ is allowed to attend to position $i$, and $M_{i,j}=-\infty$ forbids it.
\end{itemize}
$M$ induces an adjacency matrix $A \in \{0,1\}^{k\times k}$ as in Definition~\ref{def:single_layer_graph}:
\begin{equation}\label{eq:A_from_M_task}
    A_{i,j} = 1 \;\Longleftrightarrow\; M_{j,i} = 0 \text{ or } i = j \; (\text{residual connection}).
\end{equation}
\end{definition}

\begin{definition}[Dense Training Task]
\label{def:dense_training_task}
Let $\tau = (\mathcal{W}, \mathcal{V}, M)$ be a training task with adjacency matrix $A$, and let $R$ be its limiting reachability matrix given by Theorem~\ref{thm:convergence}. If
\begin{equation}\label{eq:dense_condition}
    A = R,
\end{equation}
then $\tau$ is called a dense training task.
\end{definition}

\begin{definition}[Hasse Diagram of a Task]
\label{def:task_hasse}
Let $\tau = (\mathcal{W}, \mathcal{V}, M)$ be a dense training task with index set $S = \{1,2,\dots,k\}$. By Definition~\ref{def:dense_training_task}, $M$ yields the adjacency matrix $A$ with $A=R$. On $S$, define the query relation $\preceq_\tau$ via $R$:
\begin{equation}
    \forall a,b \in S,\quad a \preceq_\tau b \;\Longleftrightarrow\; R_{b,a} = 1.
\end{equation}
By Lemma~\ref{lem:preorder}, $\preceq_\tau$ is a preorder on $S$.

Define the equivalence relation $\sim_\tau$:
\begin{equation}
    \forall a,b \in S,\quad a \sim_\tau b \;\Longleftrightarrow\; a \preceq_\tau b \text{ and } b \preceq_\tau a.
\end{equation}
The equivalence class of $a$ is
\begin{equation}
    c_{[\tau,a]} = \{b \in S \mid b \sim_\tau a\}.
\end{equation}
The quotient set is
\begin{equation}
    \mathcal{S}_\tau = S/{\sim_\tau} = \{c_{[\tau,a]} \mid a \in S\}.
\end{equation}
On $\mathcal{S}_\tau$, $\preceq_\tau$ induces the binary relation $\le_\tau$:
\begin{equation}\label{eq:induced_order}
    c_{[\tau,a]} \le_\tau c_{[\tau,b]} \;\Longleftrightarrow\; a \preceq_\tau b.
\end{equation}

By Theorem~\ref{thm:induced_partial_order}, $\le_\tau$ is a well-defined partial order. The Hasse diagram of the poset $(\mathcal{S}_\tau, \le_\tau)$ is denoted by $\mathcal{H}_\tau = (\mathcal{S}_\tau, \mathcal{E}_\tau)$, where the vertex set is the quotient set $\mathcal{S}_\tau$ and the edge set $\mathcal{E}_\tau$ is defined by the covering relation: for $c_{[\tau,a]}, c_{[\tau,b]} \in \mathcal{S}_\tau$, an edge $(c_{[\tau,a]}, c_{[\tau,b]}) \in \mathcal{E}_\tau$ exists iff
\begin{equation}
    c_{[\tau,a]} <_\tau c_{[\tau,b]} \text{ and there is no } c_{[\tau,x]} \in \mathcal{S}_\tau \text{ with } c_{[\tau,a]} <_\tau c_{[\tau,x]} <_\tau c_{[\tau,b]}.
\end{equation}
Edges are directed from the smaller node to the larger node.
\end{definition}

\begin{definition}[Input and Label on Nodes]
\label{def:node_input_label}
Let $\mathcal{H}_\tau = (\mathcal{S}_\tau, \mathcal{E}_\tau)$ be the Hasse diagram of task $\tau$. For each node $c_{[\tau,a]} \in \mathcal{S}_\tau$:
\begin{itemize}
    \item Its input is the multiset defined on $\{\mathcal{W}(b) \mid b \in c_{[\tau,a]}\}$:
    \begin{equation}
        \mathcal{I}_{c_{[\tau,a]}} = \{\!\{\mathcal{W}(b) \mid b \in c_{[\tau,a]}\}\!\},
    \end{equation}
    where $\{\!\{\cdot\}\!\}$ denotes a multiset.
    \item Its label is a function $\mathcal{L}_{c_{[\tau,a]}}$ defined on $\{\mathcal{W}(b) \mid b \in c_{[\tau,a]}\}$:
    \begin{equation}
        \forall x \in \{\mathcal{W}(b) \mid b \in c_{[\tau,a]}\},\quad
        \mathcal{L}_{c_{[\tau,a]}}(x) = \{\mathcal{V}(b) \mid b \in c_{[\tau,a]} \land \mathcal{W}(b) = x\}.
    \end{equation}
    If for some $x$, none of the corresponding $b$ lie in the domain $S_{\mathcal{V}}$ of $\mathcal{V}$, then $\mathcal{L}_{c_{[\tau,a]}}(x) = \varnothing$.
\end{itemize}
\end{definition}

\subsection{Node Equivalence}
\label{sec:node_equivalence}

To merge multiple tasks into a computation graph that can be evaluated in a single forward pass, we need to identify nodes in different Hasse diagrams that share exactly the same computational structure. We call such nodes equivalent. Two nodes are equivalent only if they correspond to each other and all their upstream queryable nodes also correspond one-to-one.

\begin{definition}[Node Equivalence]
\label{def:node_equivalence}
Let $\tau_1 = (\mathcal{W}_1, \mathcal{V}_1, M_1)$ and $\tau_2 = (\mathcal{W}_2, \mathcal{V}_2, M_2)$ be two dense training tasks on the same sample, with Hasse diagrams $\mathcal{H}_1 = (\mathcal{S}_1, \mathcal{E}_1)$ and $\mathcal{H}_2 = (\mathcal{S}_2, \mathcal{E}_2)$. For nodes $c \in \mathcal{S}_1$ and $d \in \mathcal{S}_2$, we say $c$ and $d$ are equivalent, denoted $c \approx d$, if there exists a bijection
\begin{equation}
    \phi: \{x \in \mathcal{S}_1 \mid x \le_1 c\} \;\longrightarrow\; \{y \in \mathcal{S}_2 \mid y \le_2 d\},
\end{equation}
satisfying:
\begin{enumerate}
    \item \textbf{Structure preservation:} $\phi(c) = d$, and for all $x,y \in \{x \in \mathcal{S}_1 \mid x \le_1 c\}$,
    \begin{equation}
        (x,y) \in \mathcal{E}_1 \;\Longleftrightarrow\; (\phi(x),\phi(y)) \in \mathcal{E}_2.
    \end{equation}
    \item \textbf{Input equality:} For every $x \in \{x \in \mathcal{S}_1 \mid x \le_1 c\}$, the input multiset of $x$ equals that of its image:
    \begin{equation}
        \mathcal{I}_x = \mathcal{I}_{\phi(x)}.
    \end{equation}
\end{enumerate}
\end{definition}

The definition deliberately does not require equality of labels. During training, different supervision targets are allowed for the same output, because the supervision signal does not affect the forward computation. Node equivalence concerns only the isomorphism of information-flow structures and the equality of computed values; it is independent of the specific supervision signals.

Intuitively, $c \approx d$ means that $c$ and $d$ reside in identical computational dependency structures in their respective Hasse diagrams; consequently, the results they compute are exactly the same. Hence, in a merged computation graph, the two groups of positions represented by $c$ and $d$ can be handled by a single merged node.

\subsection{Merging Tasks and the Common Supergraph}
\label{sec:task_merging}

Now consider a set of training tasks. Our goal is to merge them into a unified computation graph so that a single forward pass suffices to train all subtasks simultaneously.

\begin{definition}[Task Family and Merged Task]
\label{def:task_family}
Let $\mathcal{T} = \{\tau_1, \tau_2, \dots, \tau_m\}$ be $m$ dense training tasks on the same sample, where $\tau_k = (\mathcal{W}_k, \mathcal{V}_k, M_k)$ and its Hasse diagram is $\mathcal{H}_k = (\mathcal{S}_k, \mathcal{E}_k)$. If there exists a dense training task $\tau^*$ with Hasse diagram $\mathcal{H}^* = (\mathcal{S}^*, \mathcal{E}^*)$ such that
\begin{equation}
    \forall k \in \{1,\dots,m\},\; \forall c \in \mathcal{S}_k,\;
    \exists d \in \mathcal{S}^* \text{ with } c \approx d,
\end{equation}
then $\tau^*$ is called a \emph{merged task} of $\mathcal{T}$, and $\mathcal{H}^*$ is a \emph{common supergraph} of $\{\mathcal{H}_k\}_{k=1}^{m}$. Among all common supergraphs, the one with the fewest nodes is called the \emph{minimal common supergraph}, and the corresponding task is called the \emph{minimal merged task}.
\end{definition}

The above definition reformulates the design of parallel training as a purely graph-theoretic problem: given a set of Hasse diagrams, find their minimal common supergraph. The following theorem provides an explicit construction of the minimal common supergraph and proves its minimality.

\begin{theorem}[Minimal Common Supergraph Criterion]
\label{thm:minimal_common_supergraph}
Let $\mathcal{T} = \{\tau_1,\dots,\tau_m\}$ be a family of dense training tasks, with Hasse diagram vertex sets $\mathcal{S}_1,\dots,\mathcal{S}_m$. Let $\mathcal{U} = \bigcup_{k=1}^{m} \mathcal{S}_k$ be the union of all vertex sets. The relation $\approx$ from Definition~\ref{def:node_equivalence} is an equivalence relation on $\mathcal{U}$. Denote the set of equivalence classes by $\mathcal{Q} = \mathcal{U}/{\approx}$.

For each equivalence class $S \in \mathcal{Q}$, pick a representative. Construct the vertex set:
\begin{equation}
    \mathcal{S}^* = \{G(S) \mid S \in \mathcal{Q}\},
\end{equation}
where $G(S)$ is the chosen representative from $S$. On $\mathcal{S}^*$, define a partial order $\le^*$ as the transitive closure of the partial orders $\le_k$ of the individual $\mathcal{H}_k$, and then define the edge set $\mathcal{E}^*$ by the covering relation. Then $\mathcal{H}^* = (\mathcal{S}^*, \mathcal{E}^*)$ is the minimal common supergraph of $\mathcal{T}$.
\end{theorem}

\begin{proof}
First we verify that $\mathcal{H}^*$ is a legitimate Hasse diagram. By Definition~\ref{def:node_equivalence}, $\approx$ on $\mathcal{U}$ is reflexive, symmetric, and transitive: reflexivity holds via the identity mapping; symmetry follows because if $c \approx d$ with bijection $\phi$, then $\phi^{-1}$ witnesses $d \approx c$; transitivity follows by composing the witnessing bijections. Hence $\mathcal{U}$ admits the equivalence relation $\approx$, and the quotient set $\mathcal{Q}$ is well defined. $\mathcal{S}^*$ selects exactly one representative from each equivalence class, so $|\mathcal{S}^*| = |\mathcal{Q}|$.

Each $\mathcal{S}_k$ carries a partial order $\le_k$. If $c \le_k d$ and $c \approx c'$, $d \approx d'$ (with $c',d' \in \mathcal{S}^*$ being the representatives), the order-preservation property of the bijection in Definition~\ref{def:node_equivalence} implies $c' \le^* d'$. Thus the $\le_k$ can be consistently lifted to the quotient set $\mathcal{Q}$, and hence induce a well-defined partial order $\le^*$ on $\mathcal{S}^*$. Taking the covering edges of $\le^*$ yields a Hasse diagram, and for each $\mathcal{H}_k$, $\mathcal{H}^*$ contains its quotient embedding; therefore $\mathcal{H}^*$ is a common supergraph.

To prove minimality, suppose there exists another common supergraph $\mathcal{H}^{**} = (\mathcal{S}^{**}, \mathcal{E}^{**})$ with $|\mathcal{S}^{**}| < |\mathcal{S}^*|$. Since $\mathcal{H}^{**}$ is a common supergraph, for each node $c$ of each $\mathcal{H}_k$ there exists $b \in \mathcal{S}^{**}$ with $c \approx b$. This means every node in $\mathcal{S}^{**}$ belongs to some equivalence class $S \in \mathcal{Q}$. But $|\mathcal{S}^{**}| < |\mathcal{Q}|$; by the pigeonhole principle, there must be an equivalence class $S_0 \in \mathcal{Q}$ that contains no vertex of $\mathcal{S}^{**}$. Then $\mathcal{H}^{**}$ lacks any node equivalent to those in $S_0$, contradicting the assumption that $\mathcal{H}^{**}$ is a common supergraph. Hence $\mathcal{H}^*$ has the minimum number of nodes among all common supergraphs and is the minimal common supergraph.
\end{proof}

The minimal merged task $\tau^*$ constructed by Theorem~\ref{thm:minimal_common_supergraph} uses a dense attention mask; this was necessary to directly read off the Hasse diagram from the mask during the derivation. In practice, however, structurally identical tasks can sometimes be realized with sparser masks.

\begin{definition}[Sparse Minimal Merged Task]
\label{def:sparse_minimal_merged_task}
Let $\tau^* = (\mathcal{W}^*, \mathcal{V}^*, M^*)$ be the minimal merged task constructed according to Theorem~\ref{thm:minimal_common_supergraph}, with Hasse diagram $\mathcal{H}^* = (\mathcal{S}^*, \mathcal{E}^*)$. If there exists a training task $\tau = (\mathcal{W}, \mathcal{V}, M)$ (Definition~\ref{def:training_task}) such that:
\begin{enumerate}
    \item its limiting reachability matrix $R$ yields a Hasse diagram $\mathcal{H}_\tau$ identical to $\mathcal{H}^*$, i.e., $\mathcal{S}_\tau = \mathcal{S}^*$ and $\mathcal{E}_\tau = \mathcal{E}^*$;
    \item its mask $M$ does not satisfy the dense attention condition, i.e., $A \neq R$ (Definition~\ref{def:dense_training_task});
\end{enumerate}
then $\tau$ is called a \emph{sparse minimal merged task} of $\mathcal{T}$.
\end{definition}

The sparse minimal merged task shares exactly the same Hasse diagram structure as the dense version, and therefore expresses the same information-flow relations; the only difference is that the sparse version requires multiple Transformer layers to reach the limiting reachability matrix $R$, whereas the dense version converges in a single layer. Because the Hasse diagrams are identical, the sparse minimal merged task preserves the computational dependency structure for every subtask in the family.

\section{Case Studies}
\label{sec:case_studies}

This section applies the theoretical framework of Section~4 to concrete attention mechanism designs. We consider a length-$n$ sequence sample $\mathcal{C} = (w_1, w_2, \dots, w_n)$, where each $w_k \in \mathcal{D}$ is a token embedded with positional encoding. All masks derived in this section differ from standard attention only by a mask matrix; they can be plugged into frameworks supporting FlexAttention (e.g., Qwen3) without designing new operators.

The sections in this section are self-contained: each analyzes one attention mechanism, and notation (such as task family $\mathcal{T}$, subtask $\tau_k$, minimal merged mask $M^*$, etc.) is valid only within the respective subsection.

We analyze three attention mechanism constructions: Causal Attention (Section~\ref{sec:causal_attention}), Block Two-Stream Attention (Section~\ref{sec:block_two_stream}), and Butterfly Attention (Section~\ref{sec:three_stream}).

\subsection{Causal Attention Mask}
\label{sec:causal_attention}

Causal attention is the standard paradigm for autoregressive language modeling. We now analyze it using the theoretical tools of Section~3.

\subsubsection{Definition of the Task Family}

Let the input sequence be $\mathcal{C} = (w_1, w_2, \dots, w_n)$. Define the autoregressive task family $\mathcal{T} = \{\tau_1, \tau_2, \dots, \tau_{n-1}\}$, where task $\tau_k$ is:
\begin{equation}
    \tau_k = (\mathcal{W}_k, \mathcal{V}_k, M_k),
\end{equation}
specifically:
\begin{itemize}
    \item Input sequence $\mathcal{W}_k = (w_1, w_2, \dots, w_k)$, of length $k$.
    \item Label mapping $\mathcal{V}_k$ is defined only at position $k$: $\mathcal{V}_k(k) = w_{k+1}$.
    \item Mask $M_k$ is the standard causal mask:
    \begin{equation}
        (M_k)_{i,j} = 
        \begin{cases}
            0, & j \le i \quad (\text{position } i \text{ can attend to position } j), \\
            -\infty, & j > i.
        \end{cases}
    \end{equation}
\end{itemize}

\subsubsection{Hasse Diagram of a Single Task}

For task $\tau_k$, the mask $M_k$ yields the adjacency matrix $A_k$ (Definition~\ref{def:single_layer_graph}):
\begin{equation}
    (A_k)_{i,j} = 1 \iff j \le i.
\end{equation}
Clearly $A_k \times A_k = A_k$, so by Theorem~\ref{thm:convergence} the limiting reachability matrix $R_k = A_k$, confirming that $\tau_k$ is indeed a dense training task (Definition~\ref{def:dense_training_task}).

From $R_k$, on the index set $S_{(\tau,k)} = \{1,\dots,k\}$ define the query relation $\preceq_{(\tau,k)}$:
\begin{equation}
    j \preceq_{(\tau,k)} i \iff (R_k)_{i,j} = 1 \iff j \le i.
\end{equation}
The equivalence relation $\sim_{(\tau,k)}$ is the identity: $i \sim_{(\tau,k)} j \iff i=j$. Each position forms an equivalence class:
\begin{equation}
    c_{[\tau_k, i]} = \{i\}, \quad i=1,\dots,k.
\end{equation}
Quotient set $\mathcal{S}_{(\tau,k)} = \{c_{[\tau_k,1]}, c_{[\tau_k,2]}, \dots, c_{[\tau_k,k]}\}$.

The partial order $\le_{(\tau,k)}$ is the natural number order:
\begin{equation}
    c_{[\tau_k,i]} \le_{(\tau,k)} c_{[\tau_k,j]} \iff i \le j.
\end{equation}
Its Hasse diagram $\mathcal{H}_{(\tau,k)}$ is a directed chain of length $k$, as shown in Figure~\ref{fig:因果的哈斯图}.

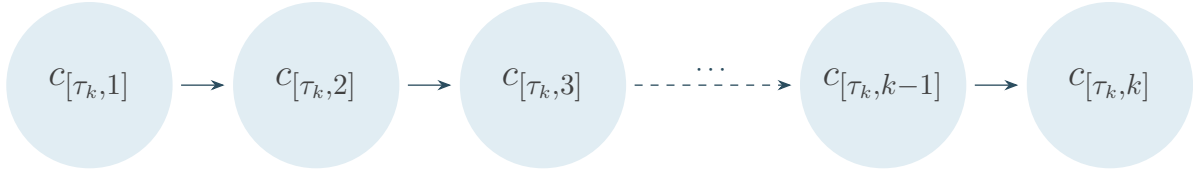
\begin{figure}[H]
    \centering
    $\vcenter{\hbox{%
        \begin{tikzpicture}[
            >=Stealth,
            vertex/.style={
                circle,
                fill=connectColor!18,
                draw=none,
                minimum size=2.2cm,
                inner sep=0pt,
                font=\sffamily\bfseries\fontsize{15}{12}\selectfont,
                text=black!75
            },
            edge/.style={
                ->,
                line width=0.5pt,
                color=connectColor!50!black,
                shorten <=3pt,
                shorten >=3pt
            },
            dashed edge/.style={
                ->, dashed,
                line width=0.5pt,
                color=connectColor!50!black,
                shorten <=3pt,
                shorten >=3pt
            }
        ]
        \node[vertex] (1) at (0,0) {$c_{[\tau_k,1]}$};
        \node[vertex] (2) at (3,0) {$c_{[\tau_k,2]}$};
        \node[vertex] (3) at (6,0) {$c_{[\tau_k,3]}$};
        \node[vertex] (k-1) at (10.5,0) {$c_{[\tau_k,k-1]}$};
        \node[vertex] (k) at (13.5,0) {$c_{[\tau_k,k]}$};

        \draw[edge] (1) -- (2);
        \draw[edge] (2) -- (3);
        \draw[dashed edge] (3) -- (k-1)
            node[midway, above, font=\itshape] {$\cdots$};
        \draw[edge] (k-1) -- (k);
        \end{tikzpicture}%
    }}$
    \caption{Hasse diagram of the autoregressive task $\tau_k$.}
    \label{fig:因果的哈斯图}
\end{figure}

Input and label on nodes (Definition~\ref{def:node_input_label}):
\begin{itemize}
    \item For $i < k$: input multiset $\mathcal{I}_{c_{[\tau_k,i]}} = \{\!\{w_i\}\!\}$, label empty.
    \item For $i = k$: input multiset $\mathcal{I}_{c_{[\tau_k,k]}} = \{\!\{w_k\}\!\}$, label function $\mathcal{L}_{c_{[\tau_k,k]}}(w_k) = \{w_{k+1}\}$.
\end{itemize}

\subsubsection{Construction of the Minimal Merged Task}

Now merge the Hasse diagrams of all $n-1$ tasks in the family $\mathcal{T}$.

According to Definition~\ref{def:node_equivalence}, node $c_{[\tau_k,i]} \in \mathcal{S}_{(\tau,k)}$ is equivalent to $c_{[\tau_\ell,j]} \in \mathcal{S}_{(\tau,\ell)}$ iff:
\begin{enumerate}
    \item there exists an order-and-edge-preserving bijection $\phi: \{x \mid x \le c_{[\tau_k,i]}\} \to \{y \mid y \le c_{[\tau_\ell,j]}\}$;
    \item the corresponding nodes have the same input multisets.
\end{enumerate}

Since all Hasse diagrams are chains and the input token $w_i$ at position $i$ is unique in the sample, equivalence depends only on the position index in the chain:

\begin{equation}
    c_{[\tau_k,i]} \approx c_{[\tau_\ell,j]} \iff i = j.
\end{equation}
Thus the equivalence classes are:
\begin{equation}
    \mathcal{Q} = \big\{ \{c_{[\tau_k,i]} \mid k > i\} \mid i = 1,2,\dots,n-1 \big\}.
\end{equation}
That is, all nodes corresponding to position $i$ across different tasks form one equivalence class.

By Theorem~\ref{thm:minimal_common_supergraph}, choosing a representative from each class gives the minimal common supergraph vertex set:
\begin{equation}
    \mathcal{S}^* = \{c_{[\tau_{n-1},1]}, c_{[\tau_{n-1},2]}, \dots, c_{[\tau_{n-1},n-1]}\}.
\end{equation}
Using the structure-preservation property (Definition~\ref{def:node_equivalence}) to add edges yields the Hasse diagram $\mathcal{H}^* = (\mathcal{S}^*, \mathcal{E}^*)$, which is a chain of length $n-1$, as shown in Figure~\ref{fig:因果的哈斯图2}.

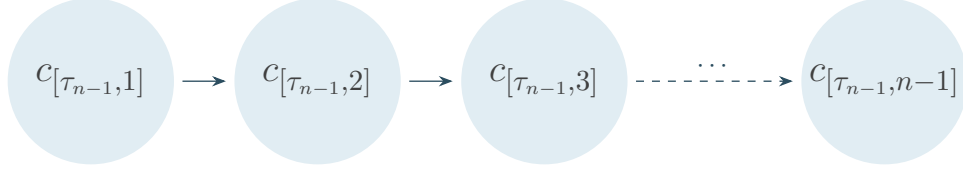
\begin{figure}[H]
    \centering
    $\vcenter{\hbox{%
        \begin{tikzpicture}[
            >=Stealth,
            vertex/.style={
                circle,
                fill=connectColor!18,
                draw=none,
                minimum size=2.2cm,
                inner sep=0pt,
                font=\sffamily\bfseries\fontsize{15}{12}\selectfont,
                text=black!75
            },
            edge/.style={
                ->,
                line width=0.5pt,
                color=connectColor!50!black,
                shorten <=3pt,
                shorten >=3pt
            },
            dashed edge/.style={
                ->, dashed,
                line width=0.5pt,
                color=connectColor!50!black,
                shorten <=3pt,
                shorten >=3pt
            }
        ]
        \node[vertex] (1) at (0,0) {$c_{[\tau_{n-1},1]}$};
        \node[vertex] (2) at (3,0) {$c_{[\tau_{n-1},2]}$};
        \node[vertex] (3) at (6,0) {$c_{[\tau_{n-1},3]}$};
        \node[vertex] (k-1) at (10.5,0) {$c_{[\tau_{n-1},n-1]}$};

        \draw[edge] (1) -- (2);
        \draw[edge] (2) -- (3);
        \draw[dashed edge] (3) -- (k-1)
            node[midway, above, font=\itshape] {$\cdots$};
        \end{tikzpicture}%
    }}$
    \caption{Hasse diagram of the minimal merged task for the autoregressive task family.}
    \label{fig:因果的哈斯图2}
\end{figure}

From the Hasse diagram, the partial order $\le_{\tau^*}$ is:
\begin{equation}
    c_{[\tau_{n-1},i]}\le_{\tau^*} c_{[\tau_{n-1},j]}\Leftrightarrow i\le j.
\end{equation}

The corresponding dense minimal merged task $\tau^* = (\mathcal{W}^*, \mathcal{V}^*, M^*)$ is:
\begin{itemize}
    \item $\mathcal{W}^* = (w_1, w_2, \dots, w_{n-1})$.
    \item $\mathcal{V}^*(i) = w_{i+1}$.
    \item $M^*$ is an $(n-1) \times (n-1)$ lower-triangular matrix of all ones ($A^* = R^*$).
\end{itemize}

This is exactly the computation actually performed by a standard autoregressive Transformer in one forward pass. The forward computation simultaneously completes $n-1$ subtasks, with intermediate representations fully shared across subtasks, achieving highly efficient training. Moreover, this proves that the standard causal attention training task is indeed the minimal merged task.

\subsection{Block Two-Stream Attention Mask}
\label{sec:block_two_stream}

This section derives the Block Two-Stream Attention: mask tokens are grouped into blocks of size $l$ with full intra-block connectivity, exploiting the two-stream property to obtain a more regular mask structure.

\subsubsection{Expansion of the Sample Set}

Let the original sequence sample be $\mathcal{C} = (w_1, w_2, \dots, w_n)$, with each $w_i \in \mathcal{D}$ embedded with positional encoding. To support block two-stream training, we augment the sample set:
\begin{equation}
    \mathcal{C}' = \mathcal{C} \cup \{m_1, m_2, \dots, m_l\},
\end{equation}
where $m_1,\dots,m_l$ are $l$ mask tokens with distinct positional encodings. They serve as fixed ``to-be-generated'' placeholders: during iterative generation, the model places $\{m_1,\dots,m_l\}$ at the end of the input and predicts the next $l$ tokens of the block.

\subsubsection{Definition of the Task Family}

Partition the sequence $\mathcal{C} = (w_1,\dots,w_n)$ into $N$ blocks of size $l$ (assuming $n = Nl$). Block $i$ contains tokens $w_{(i-1)l+1},\dots,w_{i\cdot l}$.

Define the block-generation task family $\mathcal{T} = \{\tau_1, \tau_2, \dots, \tau_{N-1}\}$, where subtask $\tau_k = (\mathcal{W}_k, \mathcal{V}_k, M_k)$ is:
\begin{itemize}
    \item $\mathcal{W}_k$ has length $(k+1)l$: the first $kl$ positions are the first $k$ real blocks, and the last $l$ positions are mask tokens. The mask tokens use $m_i$, $i\in[k,k+l)$.
    \item $\mathcal{V}_k$ is defined on the last $l$ positions: for $i=1,\dots,l$, $\mathcal{V}_k(kl+i) = w_{kl+i}$ (predicting the tokens of block $k+1$).
    \item $M_k$ is the block-causal mask of block size $l$: the indices are partitioned into $k+1$ blocks; within each block the $l$ positions are fully connected; the blocks obey a causal order, i.e., block $i$ can attend to block $j$ iff $j \le i$.
    \[
    (M_k)_{i,j} = 
    \begin{cases}
        0, & \text{if } \lceil i/l \rceil \ge \lceil j/l \rceil, \\
        -\infty, & \text{otherwise}.
    \end{cases}
    \]
\end{itemize}

\subsubsection{Hasse Diagram of a Single Task}

From $M_k$ the adjacency matrix $A_k$ is all-ones within blocks and lower-triangular across blocks. Clearly $A_k \times A_k = A_k$, hence $R_k = A_k$ and $\tau_k$ is a dense training task.

Under $R_k$, every pair of positions within the same block is bidirectionally reachable, so each block forms one equivalence-class node. Between different blocks only unidirectional reachability holds. Hence the quotient set $\mathcal{S}_{(\tau,k)}$ contains $k+1$ nodes:

\begin{equation}
\begin{aligned}
    &c_{[\tau_k,i]} &=& \{(i-1)l+1,\dots,i\cdot l\}, & i=1,\dots,k; \\
    &c_{[\tau_k,k+1]} &=& \{kl+1,\dots,(k+1)l\}.
\end{aligned}
\end{equation}
The partial order $\le_{(\tau,k)}$ is the natural order of block indices. The Hasse diagram $\mathcal{H}_{(\tau,k)}$ is a chain of length $k+1$, shown in Figure~\ref{fig:块生成的哈斯图}.
\begin{figure}[H]
    \centering
    $\vcenter{\hbox{%
        \begin{tikzpicture}[
            >=Stealth,
            vertex/.style={
                circle,
                fill=connectColor!18,
                draw=none,
                minimum size=2.2cm,
                inner sep=0pt,
                font=\sffamily\bfseries\fontsize{15}{12}\selectfont,
                text=black!75
            },
            edge/.style={
                ->,
                line width=0.5pt,
                color=connectColor!50!black,
                shorten <=3pt,
                shorten >=3pt
            },
            dashed edge/.style={
                ->, dashed,
                line width=0.5pt,
                color=connectColor!50!black,
                shorten <=3pt,
                shorten >=3pt
            }
        ]
        \node[vertex] (1) at (0,0) {$c_{[\tau_k,1]}$};
        \node[vertex] (2) at (3,0) {$c_{[\tau_k,2]}$};
        \node[vertex] (3) at (6,0) {$c_{[\tau_k,3]}$};
        \node[vertex] (k-1) at (10.5,0) {$c_{[\tau_k,k]}$};
        \node[vertex] (k) at (13.5,0) {$c_{[\tau_k,k+1]}$};

        \draw[edge] (1) -- (2);
        \draw[edge] (2) -- (3);
        \draw[dashed edge] (3) -- (k-1)
            node[midway, above, font=\itshape] {$\cdots$};
        \draw[edge] (k-1) -- (k);
        \end{tikzpicture}%
    }}$
    \caption{Hasse diagram of the block two-stream task $\tau_k$.}
    \label{fig:块生成的哈斯图}
\end{figure}
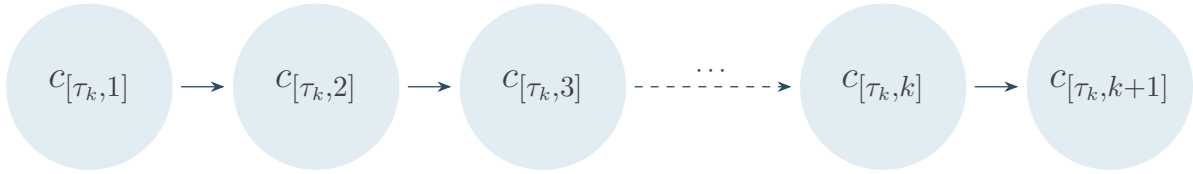

Node inputs and labels (Definition~\ref{def:node_input_label}):
\begin{itemize}
    \item Content block $c_{[\tau_k,i]}$ ($i\le k$): input multiset $\mathcal{I} = \{\!\{w_{(i-1)l+1},\dots,w_{i\cdot l}\}\!\}$, label empty.
    \item Mask block $c_{[\tau_k,k+1]}$: input multiset $\{\!\{m_{kl+1}, m_{kl+2}, \dots, m_{(k+1)l}\}\!\}$, label function $\mathcal{L}(m_{kl+i}) = \{w_{kl+i}\}$, $i=1,\dots,l$.
\end{itemize}

\subsubsection{Construction of the Minimal Merged Task}

We merge the Hasse diagrams of $\mathcal{T} = \{\tau_1,\dots,\tau_{N-1}\}$. By Definition~\ref{def:node_equivalence}, node $c_{[\tau_k,i]}$ is equivalent to $c_{[\tau_\ell,j]}$ iff there exists a structure-preserving bijection and the input multisets are equal.

\textbf{Content-block nodes}: $c_{[\tau_k,i]}$ ($k\ge i$) and $c_{[\tau_\ell,i]}$ ($\ell\ge i$) share the same input multiset (same group of real tokens, same local positions) and the same prefix chain (length $i-1$); hence they are equivalent:
\begin{equation}
    c_{[\tau_k,i]} \approx c_{[\tau_\ell,i]}, \quad \forall k,\ell \ge i.
\end{equation}

\textbf{Mask-block nodes}: $c_{[\tau_k,k+1]}$ has an upstream chain of length $k$. When $k \neq \ell$, the prefix chain lengths differ, so no structure-preserving bijection can exist; hence mask-block nodes are not equivalent to each other. Each $c_{[\tau_k,k+1]}$ forms its own equivalence class.

Equivalence classes:
\begin{equation}
    \mathcal{Q} = \big\{ \{c_{[\tau_k,i]} \mid k \ge i\} \mid i = 1,\dots,N-1 \big\}
    \;\cup\; \big\{ \{c_{[\tau_k,k+1]}\} \mid k = 1,\dots,N-1 \big\}.
\end{equation}

By Theorem~\ref{thm:minimal_common_supergraph}, selecting representatives yields the minimal common supergraph vertex set:
\begin{equation}
    \mathcal{S}^* = \{c_{[\tau_{N-1},1]},\dots,c_{[\tau_{N-1},N-1]}\}
    \;\cup\; \{c_{[\tau_1,2]},c_{[\tau_2,3]},\dots,c_{[\tau_{N-1},N]}\}.
\end{equation}
Using structure preservation to add edges, we obtain the Hasse diagram $\mathcal{H}^* = (\mathcal{S}^*, \mathcal{E}^*)$ of Block Two-Stream Attention, shown in Figure~\ref{fig:块生成的哈斯图2}.

\begin{figure}[H]
    \centering
    $\vcenter{\hbox{%
        \begin{tikzpicture}[
            >=Stealth,
            vertex/.style={
                circle,
                fill=connectColor!18,
                draw=none,
                minimum size=2.2cm,
                inner sep=0pt,
                font=\sffamily\bfseries\fontsize{15}{12}\selectfont,
                text=black!75
            },
            edge/.style={
                ->,
                line width=0.5pt,
                color=connectColor!50!black,
                shorten <=3pt,
                shorten >=3pt
            },
            dashed edge/.style={
                ->, dashed,
                line width=0.5pt,
                color=connectColor!50!black,
                shorten <=3pt,
                shorten >=3pt
            }
        ]
        \node[vertex] (1) at (0,0) {$c_{[\tau_{N-1},1]}$};
        \node[vertex] (2) at (3,0) {$c_{[\tau_{N-1},2]}$};
        \node[vertex] (3) at (6,0) {$c_{[\tau_{N-1},3]}$};
        \node[vertex] (n-1) at (10.5,0) {$c_{[\tau_{N-1},N-1]}$};

        \node[vertex] (m1) at (0+1.5,3) {$c_{[\tau_1,2]}$};
        \node[vertex] (m2) at (3+1.5,3) {$c_{[\tau_2,3]}$};
        \node[vertex] (m3) at (6+1.5,3) {$c_{[\tau_3,4]}$};
        \node[vertex] (mn-1) at (12,3) {$c_{[\tau_{N-1},N]}$};

        \draw[edge] (1) -- (2);
        \draw[edge] (2) -- (3);
        \draw[dashed edge] (3) -- (n-1)
            node[midway, above, font=\itshape] {$\cdots$};

        \draw[edge] (1) -- (m1);
        \draw[edge] (2) -- (m2);
        \draw[edge] (3) -- (m3);
        \draw[edge] (n-1) -- (mn-1);
        \end{tikzpicture}%
    }}$
    \caption{Hasse diagram of the minimal merged task for Block Two-Stream Attention.}
    \label{fig:块生成的哈斯图2}
\end{figure}
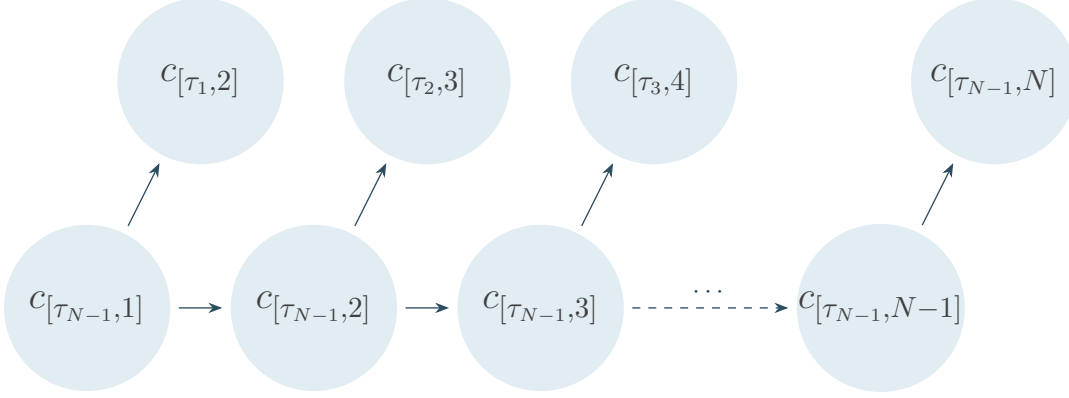

From the Hasse diagram, the partial order $\le_{\tau^*}$ is:
\begin{itemize}
    \item Main chain: $c_{[\tau_{N-1},i]} \le_{\tau^*} c_{[\tau_{N-1},j]}$ for $i,j\in[1,N-1]$ with $i\le j$.
    \item Branches: $c_{[\tau_k,k+1]}\le_{\tau^*} c_{[\tau_{N-1},k]}$ for $k\in [1,N-1]$.
\end{itemize}

The minimal merged task $\tau^* = (\mathcal{W}^*, \mathcal{V}^*, M^*)$:
\begin{itemize}
    \item Input sequence $\mathcal{W}^*$ of length $2(N-1)l$: the first $(N-1)l$ positions are the first $N-1$ real blocks; the last $(N-1)l$ positions are $N-1$ groups of mask blocks, the $t$-th group ($t=1,\dots,N-1$) containing $l$ copies of $\mathbf{m}$.
    \item Label $\mathcal{V}^*$ is defined only on the latter half: the labels for the $t$-th mask block are the tokens of real block $t+1$.
    \item Mask $M^*$ is a $2(N-1)l \times 2(N-1)l$ matrix, described by $l\times l$ blocks (block indices $1,\dots,2N-2$), as shown in Figure~\ref{fig:块双流注意力掩码}:
    \begin{itemize}
        \item Among real blocks (blocks $1..N-1$): block-causal (subblock $(i,j)=\mathbf{1} \iff j\le i$).
        \item Mask block $\to$ real block: the $i$-th mask block (index $N-1+i$) can attend to the first $i$ real blocks (subblock $(N-1+i,j)=\mathbf{1} \iff j\le i$).
        \item Among mask blocks: only self-block fully connected (subblock $(N-1+i,N-1+j)=\mathbf{1} \iff i=j$).
        \item Real block $\to$ mask block: all $-\infty$.
    \end{itemize}
\end{itemize}

\begin{figure}[H]
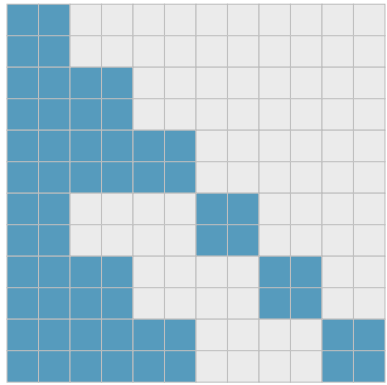

    \centering
    $\vcenter{\hbox{%
        \mask{
    1,1,0,0,0,0,0,0,0,0,0,0,
    1,1,0,0,0,0,0,0,0,0,0,0,
    1,1,1,1,0,0,0,0,0,0,0,0,
    1,1,1,1,0,0,0,0,0,0,0,0,
    1,1,1,1,1,1,0,0,0,0,0,0,
    1,1,1,1,1,1,0,0,0,0,0,0,
    1,1,0,0,0,0,1,1,0,0,0,0,
    1,1,0,0,0,0,1,1,0,0,0,0,
    1,1,1,1,0,0,0,0,1,1,0,0,
    1,1,1,1,0,0,0,0,1,1,0,0,
    1,1,1,1,1,1,0,0,0,0,1,1,
    1,1,1,1,1,1,0,0,0,0,1,1
}{12}
    }}$
    \caption{Example block two-stream attention mask matrix $M^*$.}
    \label{fig:块双流注意力掩码}
\end{figure}

Let $p = N-1$, index positions $i,j$ starting from $0$ ($0 \le i,j < 2pl$). The mask $M^*$ is given by:
\begin{equation}
    M^*_{i,j} =
    \begin{cases}
        0, & \text{if } i,j < pl \text{ and } \lfloor i/l \rfloor \ge \lfloor j/l \rfloor \\[6pt]
        0, & \text{if } i \ge pl,\; j < pl \text{ and } 
            \left\lfloor \dfrac{i-pl}{l} \right\rfloor \ge \left\lfloor \dfrac{j}{l} \right\rfloor \\[8pt]
        0, & \text{if } i,j \ge pl \text{ and } 
            \left\lfloor \dfrac{i-pl}{l} \right\rfloor = \left\lfloor \dfrac{j-pl}{l} \right\rfloor \\[8pt]
        -\infty, & \text{otherwise}.
    \end{cases}
\end{equation}

\subsection{Butterfly Attention Mask}
\label{sec:three_stream}

Butterfly Attention is another attention paradigm derived from our theoretical framework. Its goal is to enable parallel training of bidirectional attention without introducing mask tokens. This section follows the unified procedure of Section~4 to derive its minimal merged task.

\subsubsection{Expansion of the Sample Set}

Let the original sequence sample be $\mathcal{C} = (w_1, w_2, \dots, w_n)$, with each $w_k \in \mathcal{D}$ embedded with positional encoding. Training Butterfly Attention requires constructing, for each position $k$, a token $w'_k$ that replaces its own information when predicting $w_k$. Define $w'_k$ as the average of its immediate neighbors:
\begin{equation}
    w'_k = 
    \begin{cases}
        w_2, & k = 1,\\[4pt]
        \dfrac{w_{k-1} + w_{k+1}}{2}, & 1 < k < n,\\[8pt]
        w_{n-1}, & k = n.
    \end{cases}
\end{equation}
These aggregated tokens $w'_k$ have independent embeddings (different positional encodings) and are therefore distinct from the original tokens $w_k$. The augmented sample set is
\begin{equation}
    \mathcal{C}' = \mathcal{C} \cup \{w'_1, w'_2, \dots, w'_n\}.
\end{equation}

\subsubsection{Definition of the Task Family}

Define the Butterfly Attention task family $\mathcal{T} = \{\tau_1, \tau_2, \dots, \tau_n\}$, where task $\tau_k = (\mathcal{W}_k, \mathcal{V}_k, M_k)$ is:
\begin{itemize}
    \item Input sequence $\mathcal{W}_k = (w_1, \dots, w_{k-1}, w'_k, w_{k+1}, \dots, w_n)$, of length $n$. That is, at position $k$ the real token $w_k$ is replaced by the aggregated token $w'_k$.
    \item Label mapping $\mathcal{V}_k$ defined only at position $k$: $\mathcal{V}_k(k) = w_k$.
    \item Mask $M_k \in \{0,-\infty\}^{n \times n}$ defined as:
    \begin{equation}
        (M_k)_{i,j} = 
        \begin{cases}
            0, & \text{if } i \le j \le k,\\
            0, & \text{if } i = k,\\
            0, & \text{if } k \le j \le i,\\
            -\infty, & \text{otherwise}.
        \end{cases}
    \end{equation}
\end{itemize}
The mask $M_k$ means: position $k$ can attend to all positions (i.e., full visibility when $j=k$); all other positions $i \neq k$ can attend only to the interval between themselves and position $k$ (inclusive), i.e., $j$ is visible if $j \in [\min(i,k), \max(i,k)]$. Thus information flows from both sides toward the center $k$, while $k$ itself sees the complete left and right context.

\begin{figure}[H]
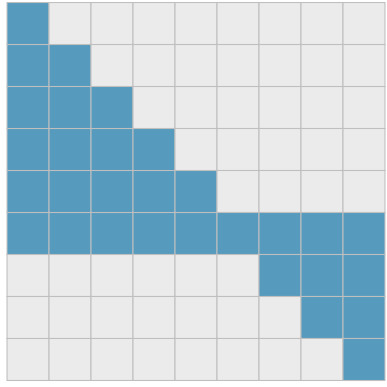

    \centering
    $\vcenter{\hbox{%
        \mask{
    1,0,0,0,0,0,0,0,0,
    1,1,0,0,0,0,0,0,0,
    1,1,1,0,0,0,0,0,0,
    1,1,1,1,0,0,0,0,0,
    1,1,1,1,1,0,0,0,0,
    1,1,1,1,1,1,1,1,1,
    0,0,0,0,0,0,1,1,1,
    0,0,0,0,0,0,0,1,1,
    0,0,0,0,0,0,0,0,1
}{9}
    }}$
    \caption{Example mask matrix $M_k$.}
    \label{fig:块生成的哈斯图}
\end{figure}

\subsubsection{Hasse Diagram of a Single Task}

From $M_k$, Definition~\ref{def:single_layer_graph} yields the adjacency matrix $A_k$:
\begin{equation}
    (A_k)_{i,j} = 1 \iff (i \le j \le k) \lor (i = k) \lor (k \le j \le i).
\end{equation}
That is, $A_k$ has ones for pairs $(i,j)$ where $i \le j \le k$ or $k \le j \le i$, and zeros elsewhere. It is easy to verify $A_k \times A_k = A_k$, so by Theorem~\ref{thm:convergence}, $R_k = A_k$ and $\tau_k$ is a dense training task (Definition~\ref{def:dense_training_task}).

Under reachability matrix $R_k = A_k$, the query relation is:
\begin{equation}
    j \preceq_{(\tau,k)} i \iff (i \le j \le k) \lor (k \le j \le i).
\end{equation}
This means:
\begin{itemize}
    \item If $i,j$ are on the same side of $k$ and $j$ is closer to $k$ than $i$, then $j \preceq_{(\tau,k)} i$ ($j$ influences $i$);
    \item If $i,j$ are on opposite sides of $k$, they cannot query each other;
    \item $k$ can query all positions, but no other position can query $k$.
\end{itemize}

Since $\sim_{(\tau,k)}$ is the identity (if $i \neq j$, we cannot have both $i \preceq j$ and $j \preceq i$), each position forms its own equivalence class:
\begin{equation}
    c_{[\tau_k, i]} = \{i\}, \quad i = 1,\dots,n.
\end{equation}
Quotient set $\mathcal{S}_{(\tau,k)} = \{c_{[\tau_k,1]}, \dots, c_{[\tau_k,n]}\}$.

The partial order $\le_{(\tau,k)}$ inherits the query relation:
\begin{equation}
    c_{[\tau_k,i]} \le_{(\tau,k)} c_{[\tau_k,j]} \iff (i \le j \le k) \lor (k \le j \le i).
\end{equation}
Thus the Hasse diagram $\mathcal{H}_{(\tau,k)}$ is a V-shaped chain centered at $c_{[\tau_k,k]}$, as shown in Figure~\ref{fig:蝴蝶注意力的哈斯图}.

\begin{figure}[H]
    \centering
    $\vcenter{\hbox{%
        \begin{tikzpicture}[
            >=Stealth,
            vertex/.style={
                circle,
                fill=connectColor!18,
                draw=none,
                minimum size=2.2cm,
                inner sep=0pt,
                font=\sffamily\bfseries\fontsize{15}{12}\selectfont,
                text=black!75
            },
            edge/.style={
                ->,
                line width=0.5pt,
                color=connectColor!50!black,
                shorten <=3pt,
                shorten >=3pt
            },
            dashed edge/.style={
                ->, dashed,
                line width=0.5pt,
                color=connectColor!50!black,
                shorten <=3pt,
                shorten >=3pt
            }
        ]
        \node[vertex] (1) at (0,2) {$c_{[\tau_k,1]}$};
        \node[vertex] (k-1) at (4,1) {$c_{[\tau_k,k-1]}$};
        \node[vertex] (k) at (7,0) {$c_{[\tau_k,k]}$};
        \node[vertex] (k+1) at (10,1) {$c_{[\tau_k,k+1]}$};
        \node[vertex] (n) at (14,2) {$c_{[\tau_k,n]}$};

        \draw[edge] (k-1) -- (k);
        \draw[edge] (k+1) -- (k);
        \draw[dashed edge] (1) -- (k-1)
            node[midway, above, font=\itshape] {$\cdots$};
        \draw[dashed edge] (n) -- (k+1)
            node[midway, above, font=\itshape] {$\cdots$};
        \end{tikzpicture}%
    }}$
    \caption{V-shaped Hasse diagram of task $\tau_k$.}
    \label{fig:蝴蝶注意力的哈斯图}
\end{figure}
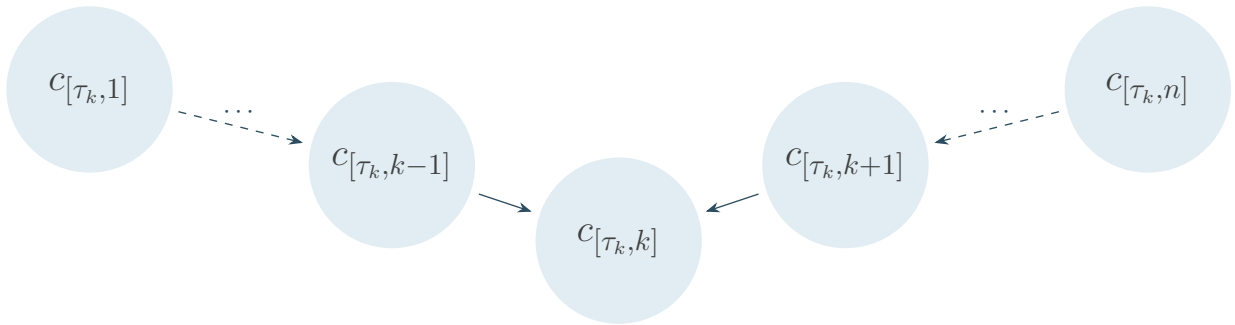

Node inputs and labels (Definition~\ref{def:node_input_label}):
\begin{itemize}
    \item For $i \neq k$: input multiset $\mathcal{I}_{c_{[\tau_k,i]}} = \{\!\{w_i\}\!\}$, label empty.
    \item For $i = k$: input multiset $\mathcal{I}_{c_{[\tau_k,k]}} = \{\!\{w'_k\}\!\}$, label function $\mathcal{L}_{c_{[\tau_k,k]}}(w'_k) = \{w_k\}$.
\end{itemize}

\subsubsection{Construction of the Minimal Merged Task}

Merge the Hasse diagrams of all $n$ tasks in $\mathcal{T}$.

By Definition~\ref{def:node_equivalence}, node equivalence requires a structure-preserving bijection and equal input multisets.

\textbf{Prefix nodes} $c_{[\tau_k,i]}$ ($i < k$): in different subtasks $\tau_k$ and $\tau_\ell$, as long as $i < k$ and $i < \ell$, these nodes have the same input $w_i$ and the same upstream chain structure (a prefix chain of length $i-1$); thus they are equivalent:
\begin{equation}
    c_{[\tau_k,i]} \approx c_{[\tau_\ell,i]}, \quad \forall k,\ell > i.
\end{equation}

\textbf{Suffix nodes} $c_{[\tau_k,j]}$ ($j > k$): similarly, nodes at the same position $j$ in different subtasks are equivalent:
\begin{equation}
    c_{[\tau_k,j]} \approx c_{[\tau_\ell,j]}, \quad \forall k,\ell < j.
\end{equation}

\textbf{Center prediction nodes} $c_{[\tau_k,k]}$: because different $k$ lead to different upstream chain structures (left chain length $k-1$, right chain length $n-k$), no structure-preserving bijection can exist when $k \neq \ell$. Hence each $c_{[\tau_k,k]}$ forms its own equivalence class.

Equivalence classes:
\begin{equation}
\begin{aligned}
    \mathcal{Q} = &\big\{ \{c_{[\tau_k,i]} \mid k > i\} \mid i = 1,\dots,n-1 \big\} \\
    &\cup \big\{ \{c_{[\tau_k,j]} \mid k < j\} \mid j = 2,\dots,n \big\} \\
    &\cup \big\{ \{c_{[\tau_k,k]}\} \mid k = 1,\dots,n \big\}.
\end{aligned}
\end{equation}

By Theorem~\ref{thm:minimal_common_supergraph}, selecting representatives yields the minimal common supergraph vertex set:
\begin{equation}
    \mathcal{S}^* = \{c_{[\tau_n,1]},\dots,c_{[\tau_n,n-1]}\}
    \cup \{c_{[\tau_1,2]},\dots,c_{[\tau_1,n]}\}
    \cup \{c_{[\tau_1,1]},c_{[\tau_2,2]},\dots,c_{[\tau_n,n]}\}.
\end{equation}
Using structure preservation to add edges, we obtain the Hasse diagram of Butterfly Attention $\mathcal{H}^* = (\mathcal{S}^*, \mathcal{E}^*)$, shown in Figure~\ref{fig:蝴蝶注意力的哈斯图2}.

\begin{figure}[H]
    \centering
    $\vcenter{\hbox{%
        \begin{tikzpicture}[
            >=Stealth,
            vertex/.style={
                circle,
                fill=connectColor!18,
                draw=none,
                minimum size=2.2cm,
                inner sep=0pt,
                font=\sffamily\bfseries\fontsize{15}{12}\selectfont,
                text=black!75
            },
            edge/.style={
                ->,
                line width=0.5pt,
                color=connectColor!50!black,
                shorten <=3pt,
                shorten >=3pt
            },
            dashed edge/.style={
                ->, dashed,
                line width=0.5pt,
                color=connectColor!50!black,
                shorten <=3pt,
                shorten >=3pt
            }
        ]
        \node[vertex] (1) at (0,0) {$c_{[\tau_{n},1]}$};
        \node[vertex] (2) at (3,0) {$c_{[\tau_{n},2]}$};
        \node[vertex] (n-2) at (7.5,0) {$c_{[\tau_{n},n-2]}$};
        \node[vertex] (n-1) at (10.5,0) {$c_{[\tau_{n},n-1]}$};

        \node[vertex] (m1) at (0-1.5,3) {$c_{[\tau_1,1]}$};
        \node[vertex] (m2) at (3-1.5,3) {$c_{[\tau_2,2]}$};
        \node[vertex] (m3) at (6-1.5,3) {$c_{[\tau_3,3]}$};
        \node[vertex] (mn-1) at (10.5-1.5,3) {{$c_{[\tau_{n-1},n-1]}$}};
        \node[vertex] (mn) at (13.5-1.5,3) {$c_{[\tau_{n},n]}$};

        \node[vertex] (f1) at (0,6) {$c_{[\tau_{1},2]}$};
        \node[vertex] (f2) at (3,6) {$c_{[\tau_{1},3]}$};
        \node[vertex] (f3) at (6,6) {$c_{[\tau_{1},4]}$};
        \node[vertex] (fn) at (10.5,6) {$c_{[\tau_{1},n]}$};

        \draw[edge] (1) -- (2);
        \draw[dashed edge] (2) -- (n-2)
            node[midway, above, font=\itshape] {$\cdots$};
        \draw[edge] (n-2) -- (n-1);

        \draw[edge] (f2) -- (f1);
        \draw[edge] (f3) -- (f2);
        \draw[edge] (fn) -- (mn-1);
        \draw[dashed edge] (fn) -- (f3)
            node[midway, above, font=\itshape] {$\cdots$};

        \draw[edge] (1) -- (m2);
        \draw[edge] (2) -- (m3);
        \draw[edge] (n-2) -- (mn-1);
        \draw[edge] (n-1) -- (mn);
            
        \draw[edge] (f3) -- (m3);
        \draw[edge] (f2) -- (m2);
        \draw[edge] (f1) -- (m1);
        \end{tikzpicture}%
    }}$
    \caption{Hasse diagram of the minimal merged task for Butterfly Attention.}
    \label{fig:蝴蝶注意力的哈斯图2}
\end{figure}
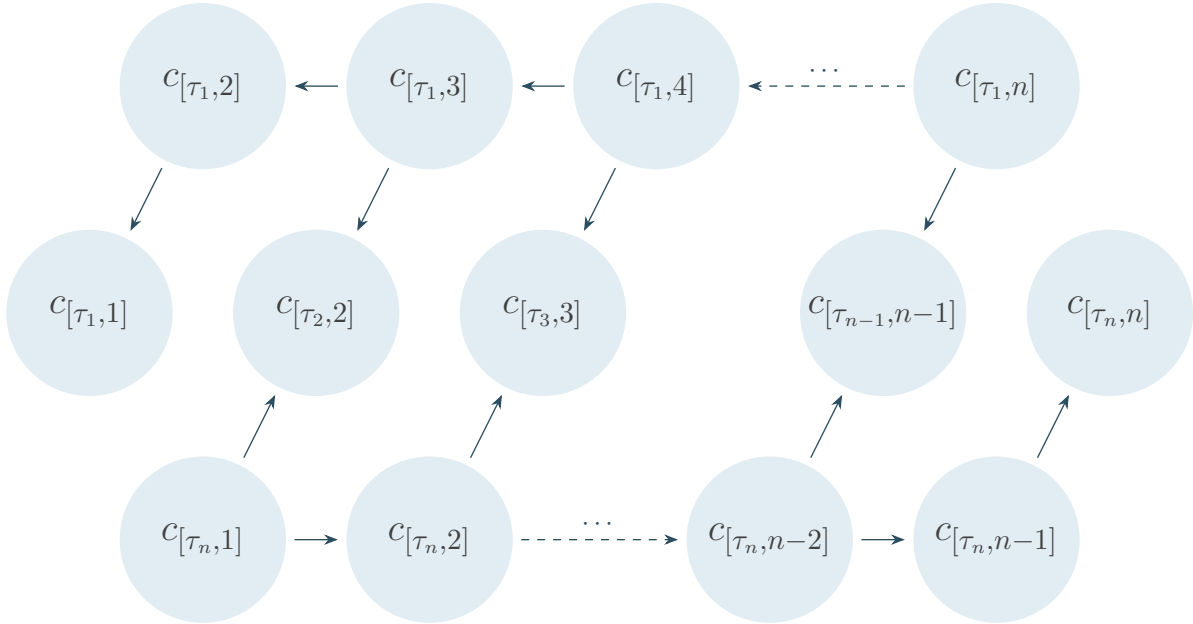

From the Hasse diagram, the partial order $\le_{\tau^*}$ is:
\begin{itemize}
    \item Forward chain: $c_{[\tau_n,i]} \le_{\tau^*} c_{[\tau_n,j]}$ for $i,j\in[1,n-1]$ with $i\le j$.
    \item Backward chain: $c_{[\tau_1,i]} \ge_{\tau^*} c_{[\tau_1,j]}$ for $i,j\in[2,n]$ with $i\le j$.
    \item Prediction flows: $c_{[\tau_n,i-1]} \le_{\tau^*} c_{[\tau_i,i]} \ge_{\tau^*} c_{[\tau_1,i]}$ for $i\in [2,n-1]$. At the endpoints, $c_{[\tau_1,1]} \ge_{\tau^*} c_{[\tau_1,2]}$ and $c_{[\tau_n,n-1]} \le_{\tau^*} c_{[\tau_n,n]}$.
\end{itemize}

\subsubsection{Minimal Merged Task}

The corresponding minimal merged task $\tau^* = (\mathcal{W}^*, \mathcal{V}^*, M^*)$:
\begin{itemize}
    \item Input sequence $\mathcal{W}^*$ of length $3n-2$:
    \begin{equation}
        \mathcal{W}^* = [w_1,\dots,w_{n-1},\; w_2,\dots,w_n,\; w'_1,\dots,w'_n].
    \end{equation}
    That is, the first $n-1$ positions are the forward sequence, the middle $n-1$ positions are the backward sequence, and the last $n$ positions are the aggregated-token sequence.
    
    \item Label $\mathcal{V}^*$ defined only on the last $n$ positions: $\mathcal{V}^*(2n-2+i) = w_i$, $i=1,\dots,n$.
    
    \item Mask $M^*$ is a square matrix of order $3n-2$, shown in Figure~\ref{fig:三流注意力掩码}. Its closed-form expression is (indices starting from $0$):
    \begin{equation}
        M^*_{i,j} =
        \begin{cases}
            0, & 0 \le i < n-1,\; 0 \le j \le i \quad \\
            0, & n-1 \le i < 2n-2,\; i \le j < 2n-2 \quad \\
            0, & 2n-2 \le i < 3n-2,\; 
            \begin{array}{l}
                {(j < i - 2n + 2) \;\lor}\\ (i-n+1 < j < 2n-2) \;\lor\; (j = i)
            \end{array} \\[4pt]
            & \quad \\[12pt]
            -\infty, & \text{otherwise}.
        \end{cases}
    \end{equation}
\end{itemize}

\begin{figure}[H]
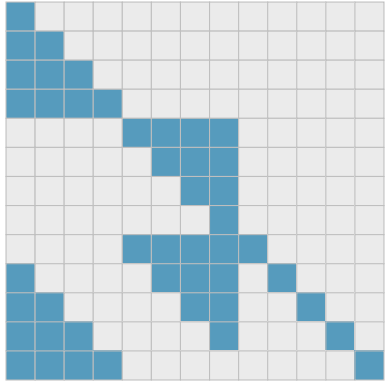

    \centering
    $\vcenter{\hbox{%
        \mask{
    1,0,0,0,0,0,0,0,0,0,0,0,0,
    1,1,0,0,0,0,0,0,0,0,0,0,0,
    1,1,1,0,0,0,0,0,0,0,0,0,0,
    1,1,1,1,0,0,0,0,0,0,0,0,0,
    0,0,0,0,1,1,1,1,0,0,0,0,0,
    0,0,0,0,0,1,1,1,0,0,0,0,0,
    0,0,0,0,0,0,1,1,0,0,0,0,0,
    0,0,0,0,0,0,0,1,0,0,0,0,0,
    0,0,0,0,1,1,1,1,1,0,0,0,0,
    1,0,0,0,0,1,1,1,0,1,0,0,0,
    1,1,0,0,0,0,1,1,0,0,1,0,0,
    1,1,1,0,0,0,0,1,0,0,0,1,0,
    1,1,1,1,0,0,0,0,0,0,0,0,1
}{13}
    }}$
    \caption{Example Butterfly Attention mask matrix $M^*$.}
    \label{fig:三流注意力掩码}
\end{figure}

\subsection{Section Summary}

In this section, we applied the theoretical framework established in Section~3 and~4. We analyzed causal attention and proved that it constitutes the minimal merged task. We then examined the block-generation task family and the bidirectional attention task family, and constructed two novel attention mechanisms: Block Two-Stream Attention and Butterfly Attention. For each attention mechanism, the same unified analysis and construction procedure was followed:
\begin{enumerate}
    \item Define the task family $\mathcal{T}$ (Definition~\ref{def:task_family});
    \item Derive the Hasse diagram of a single subtask (Definition~\ref{def:task_hasse});
    \item Compute node equivalence relations and equivalence classes (Definition~\ref{def:node_equivalence});
    \item Construct the minimal merged task via the minimal common supergraph criterion (Theorem~\ref{thm:minimal_common_supergraph});
    \item Derive the corresponding mask matrix.
\end{enumerate}

These three case studies demonstrate that the proposed framework can systematically produce training schemes and mask matrices starting from the definition of a task family, providing a mathematical foundation and a standard derivation pipeline for designing novel information-flow configurations of attention mechanisms.

\section{Conclusion}
\label{sec:conclusion}

This paper systematically investigates the problem of designing training tasks for attention mechanisms from two complementary directions. The first direction models the information flow during attention training as a Hasse diagram: we proved that the information flow of a multi-layer Transformer, given sufficient depth, converges to a stable partial-order structure, which we precisely characterize by a Hasse diagram. The second direction reduces the design of parallel training tasks to the problem of finding a minimal common supergraph of Hasse diagrams: given a set of training tasks, their optimal merged scheme is equivalent to the minimal common supergraph of their Hasse diagrams, a correspondence rigorously guaranteed by the minimal common supergraph criterion.

By combining these two directions, we obtain a design-capable theoretical framework. Mask design for attention mechanisms is no longer an empirical trial-and-error process; instead, starting from the definition of a task family, one derives the Hasse diagrams, computes equivalence classes, constructs the minimal common supergraph, and finally obtains the optimal mask in a fully constructive manner.

Currently, the framework systematically handles only the merging of dense training tasks. For sparse attention design, for the sake of theoretical simplicity, our tools support first merging the task family into a dense minimal common supergraph and then refining it into a sparse version (Definition~\ref{def:sparse_minimal_merged_task}). This design strategy mirrors the way sparse versions of causal attention are currently constructed.

This work opens new research directions for attention mechanism design. Future work includes:

\begin{enumerate}
    \item Analyzing mainstream linear attention mechanisms to verify whether their information flow also gives rise to a partial-order structure. We have shown that the information flow of standard Transformers forms a partial order, but when model architectures employ linear attention such as KDA\ \cite{kimiteam2025kimilinearexpressiveefficient} or GDN\ \cite{yang2025gateddeltanetworksimproving}, our framework cannot be directly applied. Extending the framework to accommodate such attention variants is a natural next step.
    
    \item Integrating the Hasse diagram framework with concrete training objectives to design further attention mechanisms. This paper has demonstrated three designs---causal, Block Two-Stream, and Butterfly---but the space of task-family definitions is far from exhausted. Future work on novel attention mechanisms can focus on designing better task families, particularly for MAE-like scenarios, and empirically validating their effectiveness.
\end{enumerate}
\bibliographystyle{unsrt}
\bibliography{references}  

\end{document}